\documentclass{article}
\usepackage{config_en}
\usepackage{hyperref}

\title{Configuration Space Singularities of The Delta Manipulator}
\date{\today}
\author{Marc Diesse\\ Faculty of Mechanics and Electronic, Hochschule Heilbronn }

\begin{document}
\maketitle
\begin{abstract}
We investigate the configuration space of the Delta-Manipulator, identify 24 points in the configuration space, where the Jacobian of the Constraint Equations looses rank and show, that these are not manifold points of the Real Algebraic Set, which is defined by the Constraint Equations.
\end{abstract}

\section{Introduction}
	The study of mechanism singularities is an active area of research, which started its formalization with papers by Gosselin, Zlatanov, Liu and Park \cite{gosselin:singularities}, \cite{zlatanov:singularities}, \cite{zlatanov:singularities_inst}, \cite{park:singularities}, \cite{park:manipulability}, \cite{liu:singularities}. Much work has been done trying to build a mathematical foundation for
	categorizing and comparing singular configurations, especially by Müller~\cite{mueller:singularities}.
	We hope to to add to this development by providing a simple mathematical formulation of manipulators and demonstrate its usefulness with 
	the classification of the configuration space singularities of the delta-platform.
	
	In addition we want to point out several points where confusion might arise and where future work could help to clarify local structure of real algebraic sets, which represent the configuration space of mechanism.

\section{Algebraic Preliminaries}\label{sec:prelim}
To start the discussion of configuration space singularities we need some definitions, to express the notion of configuration spaces and singularities.

\begin{definition}\label{def:singularity_algebraic}
	Let $X = \algv(f_1, \ldots f_k) \subset \R^n$ the zero-set of $k$ polynomials $f_1, \ldots f_k \in \rpn$. 
	It is $p \in X$ a \defm{singularity} of $X$, if 
	\[
	\rank \bigg[ \partial_j \, g_i \bigg]_{\substack{i=1, \ldots, k \\ j=1,\ldots, n}} < n-d,
	\]
	where $g_1, \ldots g_k$ generate the ideal of $X$, i.e. the set of all polynomials, which vanish on $X$, and $d$ is the local dimension of $X$ at $p$, which can be defined in a variety of equivalent ways (see  e.g. \cite{bochnak:real_alg_geom}).
\end{definition}

\begin{definition}\label{def:singularity_geometric}
	For any set $X$ and point $p \in X$, we say, that $p$ is a \defm{manifold point} of $X$, if for an euclidean neighbourhood $N$ of $p$, $N \cap X$ is
	an embedded submanifold of $\R^n$. Equivalently $p$ is a manifold point of $X$, iff for a neigbourhood $N$ of $p$ and a choice of linear coordinates 
	$X \cap N$ is the graph of an analytic function.
\end{definition}%
The reason we have two separate concepts of singularity is the following well known example:
\begin{ex}
Figure~\ref{fig:smooth_curve_example} shows the algebraic curve $X := \algv(f)$, with $f(x,y) = y^3 + x^2\,y - x^4 \in \R[x,y]$. Since $f(x,y)$ is irreducible and there exist nonsingular points $p \in X$, it is $\algi(X) = \langle f \rangle$ according to \cite[Theorem~4.5.1]{bochnak:real_alg_geom} Then the origin is a singularity of $X$, as stated by definition~\ref{def:singularity_algebraic}. 
But we can factorize $f$ in the ring of analytic function germs at the origin:
\[
f = (x^2 - y\,(1 + \sqrt{1 + y})) \cdot (x^2 + y\,(1 + \sqrt{1 + y})).
\]
Since $x^2 - y\,(1 - \sqrt{1 + y}) > 0 $, for $(x,y) \ne (0,0)$ close to the origin, it is 
\[
X \cap B_{\eps}(0) = \algv(x^2 - y\, (1 + \sqrt{1 + y})) \cap B_{\eps}(0),
\]
for an $\eps > 0$. Hence the origin is a manifold point of $X$ because $\algv(x^2 - y\,(1 + \sqrt{1 + y}))$ is an analytical manifold according to the analytic implicit function theorem.

\begin{remark}
In the case of complex algebraic sets, we don't have this problem, since here a point $p$ of an algebraic set $V$ is a singularity iff it is not a complex manifold point of $V$.	
\end{remark}

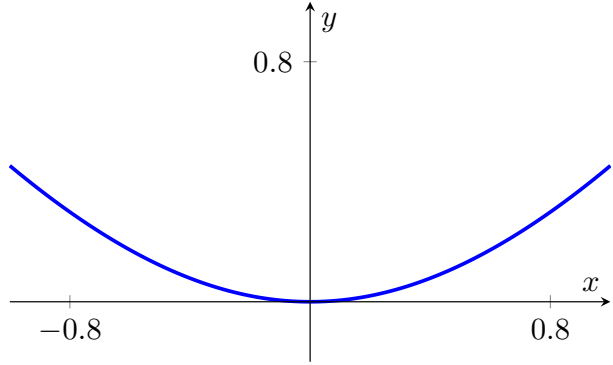
\begin{figure}
\centering
\resizebox{8cm}{!}{%
 \begin{tikzpicture}
 \begin{axis}[xmin=-1,xmax=0.5*2,ymin=-0.2, ymax=1, xtick={-0.8,0.8}, ytick={-0.8,0.8}, unit vector ratio*=1 1 1, axis lines=center, axis on top, xlabel={$x$}, ylabel={$y$}]
 \addplot[very thick, blue] file {plot_data_2xy2};
 \end{axis}
 \end{tikzpicture}}
\caption{the zero set of $y^3 + 2\,x^2\,y - x^4$}
\label{fig:smooth_curve_example}
 \end{figure}
\end{ex}
\section{Formal Manipulators}
In order to talk precisely about configuration spaces, we formalize the notion of manipulator:

\begin{definition}
	A formal \defm{manipulator} is a tupel $\CM = (X,\CA,g)$, with 
	$X \subset \R^s$ a real algebraic set, a choice $\CA = \{a_1, \ldots, a_{t_1}, c_1, \ldots c_{t_2}  \}$, with $a_i \colon X \to \R$ the projection on a coordinate, $c_j \colon X \to \R^2$ the projection on two coordinates $(x_i,x_j)$, if $x_i^2 + x_j^2 - r \in \algi(X)$, for a $r \in \R$. Besides, let $g \colon X \to \SER$ be a regular mapping. We call:
	\begin{itemize}
		\item [(i)]
		$X$ the \defm{Configuration Space}, 
		\item [(ii)]
		$\CA$ the set of \defm{actuators} of $\CM$. 
		\item [(iii)]
		$g$ the \defm{forward kinematic} of the manipulator $\CM$.
		\item [(iv)]
		The semialgebraic set $g(X)$, the \defm{workspace} of $\CM$.
	\end{itemize}
\end{definition}

All mechanism, which comprises rigid bodies, spherical, revolute and prismatic joints and where a subset of the prismatic or revolute joints is actuated, can be expressed with this formalism.

Now it is evident how to define and categorize the different notions of kinematic singularities. We follow the nomenclature of \cite{liu:singularities}.
\begin{definition}\label{def:kinematic_singularity}\hfill
\begin{itemize}
\item[(i)] 
A point $x \in X$, which is not a manifold point of $X$, we denote as \defm{configuration space singularity (CSS)}, and set
$\preg{X} \coloneqq \{ x \in X \mid \text{$x$ is a manifold point of $X$} \,\}$.

\item[(ii)] 
A point $x \in \preg{X}$, where the differential $\mathrm{d}g_x$ is not of full rank, is called an \defm{endeffector singularity (EES)}

\item[(iii)]
A point, $x \in \preg{X}$, where $(a_1, \ldots, a_{t_1}, \ldots, \rho_1 \circ c_1, \ldots, \rho_{t_2} \circ c_{t_2})$, or any subset of this set is not a chart about $x \in \preg{X}$, for all charts $\rho_i$ of $S^1$, is called an \defm{actuator singularity (AS)}.
\end{itemize}
\end{definition}

To showcase this formalism we cite the example of the crank-slider-mechanism which is quite stereotypical \cite{gosselin:singularities},\cite{zlatanov:singularities},\cite{park:singularities} in this context:
\begin{ex}
	The crank slider with joints A,B,C is depicted in the following figure:
	
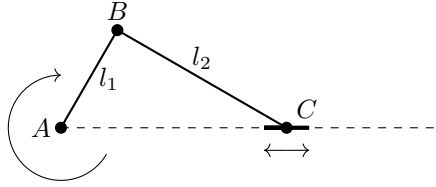
\begin{center}
	\begin{tikzpicture}
	[joint/.style={circle,inner sep=0pt, fill=black, minimum size=1.6mm},
	point/.style={circle,inner sep=0pt, draw, minimum size=1.6mm},
	sjoint/.style={circle,inner sep=0pt, fill=black, minimum size=1mm},
	laenge/.style={midway,inner sep=1.8pt},
	every node/.style={scale=1}]

	\draw[-,thick] (0,0) node[left,fill=white] {$A$} -- node[right] {$l_1$} +(60:1.5) coordinate (b) node[above,fill=white] {$B$} -- node[above] {$l_2$} (3,0) node[above right,fill=white] {$C$};
	\draw[-,dashed] (0,0) -- +(5,0);
	\draw[-,ultra thick] (3,0) -- +(-0.3,0) -- +(0.3,0);
	\draw[->] (-30: 0.7) arc[start angle=-30, end angle=-270, radius=0.7];
	\draw[->] (3,-0.3) -- +(-0.3,0); 
	\draw[->] (3,-0.3) -- +(0.3,0);

	\node at (0,0) [joint] {};
	\node at (b) [joint] {};
	\node at (3,0) [joint] {}; 
	
	\end{tikzpicture}
	
\end{center}
We denote the cartesian coordinates of $B$ with $x_B,y_B$, and the $x$-coordinate of $C$ with $x_C$.
Then the crank slider is $\CM = (X,\CA,g)$ with the configuration space $X = V(f_1, f_2) \subset \R^3$, where $f_1,f_2 \in \R[x_B,y_B,x_C]$ with
\begin{align*}
f_1 & = x_B^2 + y_B^2 - l_1^2,\\
f_2 & = (x_C - x_B)^2 + y_B^2 - l_2^2.
\end{align*}
In addition we set for the usual version of the crank slider, $\CA = \{a\}$, with $a(x_B,y_B,x_C) = x_C$ 
and the forward kinematic:
\begin{align*}
g \colon X & \to S^1 \cong \SO(2,\R) \subset \SER \\
(x_B,y_B,x_C) &\mapsto (x_B,y_B)
\end{align*}
If $l_1 \ne l_2$ we can easily check, that $f_1,f_2$ and the determinant of the Jacobian of $(f_1,f_2)$ has no common zeros and the crank slider
has therefore no singularities in the configuration space.
For $l_1 = l_2 =: l$ we have the two analytical paths
\[
\gamma_1(t) 
\coloneqq
\begin{pmatrix}
l\,\cos(t)\\
l\,\sin(t)\\
2\,l\,\cos(t)
\end{pmatrix},
\quad
\gamma_2(t)
\coloneqq
\begin{pmatrix}
l\,\cos(t)\\
l\,\sin(t)\\
0
\end{pmatrix}
\]
with $\gamma_1, \gamma_2 \subset X$ and $\dim \langle \gamma_1'(0), \gamma_2'(0) \rangle = 2$. Hence $\gamma_1(0) = \gamma_2(0) = (0,l,0)$ cannot be a manifold point of $X$, since $\dim X = 1$.
\end{ex}

\section{The Delta Manipulator}
The Delta Manipulator (figure~\ref{fig:delta}) is a type of parallel robot which consists of three identical limbs carrying a platform which serves as positioning device. 
 It was invented in the early 1980s by a research team led by Reymond Clavel and described in his Ph.D. Thesis~\cite{clavel:delta}.

 In most realizations (like the Fanuc M1) each limb comprises of a solid upper arm connected to the base with revolute joints and attached to each upper arm a parallelogram-linkage with spherical joints, which enables the end of the lower arms to travel on a spherical surface around the tip of the upper arm. 
 Both the joint-connections to the ground and to the moving platform are usually placed at the vertices of an equilateral triangle to achieve a symmetric design.

 In contrast to our illustration, the actual plattform is mounted upside down in almost all applications, in order to 
perform pick and place tasks.

\begin{figure}[h]
	\centering
\begin{tikzpicture}[line join=round]
\draw(1.349,-.156)--(1.327,-.156)--(1.307,-.151)--(1.289,-.141)--(1.274,-.127)--(1.263,-.11)--(1.257,-.09)--(1.255,-.067)--(1.257,-.044)--(1.264,-.02)--(1.275,.003)--(1.29,.025)--(1.308,.044)--(1.329,.06)--(1.351,.072)--(1.374,.08)--(1.396,.083)--(1.418,.081)--(1.437,.074)--(1.454,.063)--(1.467,.048);
\draw(-1.255,-.068)--(-1.256,-.053)--(-1.259,-.038)--(-1.263,-.022)--(-1.27,-.007)--(-1.278,.008)--(-1.288,.022)--(-1.299,.035)--(-1.312,.047)--(-1.325,.057)--(-1.339,.066)--(-1.354,.073)--(-1.369,.078)--(-1.383,.081)--(-1.398,.083)--(-1.412,.082)--(-1.426,.079)--(-1.438,.074)--(-1.449,.067)--(-1.459,.058)--(-1.467,.048);
\draw(1.982,-.377)--(1.456,.061);
\draw(-1.982,-.377)--(-1.456,.061);
\draw(-1.754,-.53)--(-1.281,-.135);
\draw(1.754,-.53)--(1.281,-.135);
\filldraw[fill=white](1.737,-.413)--(1.299,-.047)--(-1.299,-.047)--(-1.737,-.413)--(-.438,-1.497)--(.438,-1.497)--cycle;
\draw(-1.683,-.155)--(-2.202,1.432);
\draw(1.683,-.155)--(2.202,1.432);
\draw(-1.518,-.23)--(-.152,-.695);
\draw(2.008,-.444)--(2.003,-.481)--(1.986,-.518)--(1.961,-.55)--(1.93,-.575)--(1.894,-.591)--(1.859,-.595)--(1.827,-.587)--(1.802,-.569)--(1.786,-.541)--(1.781,-.507)--(1.786,-.47)--(1.802,-.434)--(1.827,-.401)--(1.859,-.376)--(1.894,-.36)--(1.93,-.356)--(1.961,-.364)--(1.986,-.382)--(2.003,-.41)--(2.008,-.444);
\draw(-2.008,-.444)--(-2.003,-.481)--(-1.986,-.518)--(-1.961,-.55)--(-1.93,-.575)--(-1.894,-.591)--(-1.859,-.595)--(-1.827,-.587)--(-1.802,-.569)--(-1.786,-.541)--(-1.781,-.507)--(-1.786,-.47)--(-1.802,-.434)--(-1.827,-.401)--(-1.859,-.376)--(-1.894,-.36)--(-1.93,-.356)--(-1.961,-.364)--(-1.986,-.382)--(-2.003,-.41)--(-2.008,-.444);
\draw(-1.807,-.574)--(-1.754,-.53);
\draw(1.807,-.574)--(1.754,-.53);
\draw(-2.071,1.542)--(-2.334,1.323);
\draw(-2.071,1.542)--(-.703,2.357);
\filldraw[color=black,fill=white](-2.071,1.542) circle (2.5pt);
\draw(2.334,1.323)--(2.071,1.542);
\draw(2.071,1.542)--(.703,2.357);
\filldraw[color=black,fill=white](2.071,1.542) circle (2.5pt);
\filldraw[fill=white](-1.737,-.528)--(-.438,-1.612)--(-.438,-1.497)--(-1.737,-.413)--cycle;
\filldraw[fill=white](.438,-1.612)--(1.737,-.528)--(1.737,-.413)--(.438,-1.497)--cycle;
\draw(2.334,1.323)--(.966,2.1);
\filldraw[color=black,fill=white](2.334,1.323) circle (2.5pt);
\filldraw[color=black,fill=white](-2.334,1.323) circle (2.5pt);
\draw(-2.334,1.323)--(-.966,2.1);
\filldraw(-.152,-.695) circle (1.5pt);
\filldraw[color=black,fill=white](-.703,2.357) circle (2.5pt);
\filldraw[color=black,fill=white](.703,2.357) circle (2.5pt);
\filldraw[fill=white](.966,2.157)--(.703,2.377)--(-.703,2.377)--(-.966,2.157)--(-.263,1.57)--(.263,1.57)--cycle;
\draw(-.835,2.267)--(0,2.035);
\filldraw[fill=white](.263,1.455)--(.966,2.042)--(.966,2.157)--(.263,1.57)--cycle;
\filldraw[fill=white](-.966,2.042)--(-.263,1.455)--(-.263,1.57)--(-.966,2.157)--cycle;
\filldraw[color=black,fill=white](-.966,2.1) circle (2.5pt);
\filldraw[color=black,fill=white](.966,2.1) circle (2.5pt);
\filldraw[fill=white](-.438,-1.612)--(.438,-1.612)--(.438,-1.497)--(-.438,-1.497)--cycle;
\filldraw(0,2.035) circle (1.5pt);
\filldraw[fill=white](.526,-1.503)--(-.526,-1.503)--(-.526,-1.742)--(.526,-1.742)--cycle;
\draw(0,-1.56)--(0,-.406);
\filldraw[fill=white](-.263,1.455)--(.263,1.455)--(.263,1.57)--(-.263,1.57)--cycle;
\draw(-.263,-.406)--(-.263,1.512);
\draw(.263,-.406)--(.263,1.512);
\draw(-.263,-.406)--(.263,-.406);+
\filldraw[color=black,fill=white](.263,1.512) circle (2.5pt);
\filldraw[color=black,fill=white](-.263,1.512) circle (2.5pt);
\filldraw[color=black,fill=white](.263,-.406) circle (2.5pt);
\filldraw[color=black,fill=white](-.263,-.406) circle (2.5pt);
\node[anchor=west] at (0,2.035) {\scriptsize{$P$}};
\node[anchor=north, yshift=1.1pt] at (-.417,2.151) {\scriptsize{$r_2$}};
\node[anchor=north, yshift=0.7pt] at (-.835,-.462) {\scriptsize{$r_1$}};
\node[anchor=east, xshift=1pt] at (1.943,.639) {\scriptsize{$a$}};
\node[anchor=south,yshift=2pt] at (1.519,1.792) {\scriptsize{$b$}};
\end{tikzpicture}
\caption{Illustration of the Delta Manipulator}\label{fig:delta}
\end{figure}
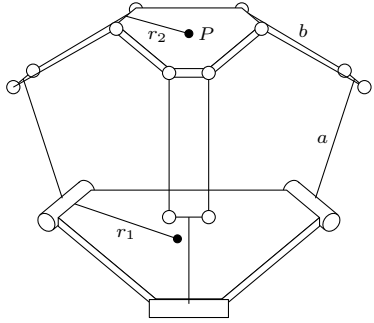
In order to investigate configuration space singularities of the delta manipulator, we want to formalize the design sketched above with parameters $a$, $b$, $r_1$, $r_2$.
So we define $M_{a,b,d} = (X_{a,b,d},\CA,g)$, $d \coloneqq r_1 - r_2$,  where the configuration space $X_{a,b,d} = \algv( \{ s_i,c_j,l_k \mid i,j=1,2,3 \text{ and } k=1,\ldots,6 \})$ is given by the following polynomials
\begin{equation}\label{eq:delta_system_1}
\begin{aligned}
s_1 &\coloneqq x_1^2 + y_1^2 + z_1^2 - b^2,\\
s_2 &\coloneqq x_2^2 + y_2^2 + z_2^2 - b^2,\\
s_3 &\coloneqq x_3^2 + y_3^2 + z_3^2 - b^2,\\
c_1 &\coloneqq ca_1^2 + sa_1^2 - a^2,\\
c_2 &\coloneqq ca_2^2 + sa_2^2 - a^2,\\
c_3 &\coloneqq ca_3^2 + sa_3^2 - a^2
\end{aligned}
\qquad
\begin{aligned}
\begin{pmatrix}
l_1\\
l_2\\
l_3
\end{pmatrix}
& = v_1 - A\,v_2, \\
\begin{pmatrix}
l_4 \\
l_5\\
l_6
\end{pmatrix}
& = v_1 - A^{-1}\,v_3
\end{aligned}
\end{equation}
where
\begin{equation*}
v_i \coloneqq 
\begin{pmatrix}
d + ca_i + x_i \\
y_i\\
z_i+sa_id
\end{pmatrix},
\qquad 
A \coloneqq
\begin{pmatrix}
-\frac{1}{2} & -\frac{\sqrt{3}}{2} & 0 \\
\phantom{-}\frac{\sqrt{3}}{2} & -\frac{1}{2} & 0 \\
\phantom{-} 0 & \phantom{-} 0 & 1
\end{pmatrix} \in \SOR,
\end{equation*}
with the 15 variables
\[
x_1,y_1,z_1,x_2,y_2,z_2,x_3,y_3,z_3,ca_1,sa_1,ca_2,sa_2,ca_3,sa_3.
\]
Although we are primarily interested in the configuration space of the delta manipulator, for completeness sake, we set
$\CA = \{c_1,c_2,c_3\}$, where $c_i$ is the projection on $(ca_i,sa_i)$ and
\begin{align*}
g \colon X_{a,b,d} & \to \R^3 \subset \SER \\
(x_i,y_i,z_i,ca_j,sa_j) & \mapsto v_1.
\end{align*}

We now collect all the polynomials $s_i,c_j,l_k$ in a polynomial map $F \colon \R^{15} \to \R^{12}$ and can formulate our main results: 
\begin{thm}\label{thm:delta_jacobi}
	Let $a,b,d \in \R^+$, with $a > d$. The dihedral group $D_3$ acts on $X_{a,b,d}$, which restricts to a group action on
	\[
	S_{a,b,c} \coloneqq \{\, x \in X_{a,b,d} \mid \rank DF(p_i) < 12 \, \}
	\] 
	There exists $24$ points $p_i \in S_{a,b,d}$ and $D_3$ acts freely on $P_{a,b,d} := \{ \, p_i \mid i=1,\ldots, 24 \}$. Four representatives of the orbits in $P_{a,b,d}/D_3$ are given in table~\ref{table:cs_singularities}.
\end{thm}

\begin{thm}\label{thm:delta_jacobi_sv}
For the choice $a=3$, $b=5$, $d=0.5$, it is $X_{3,5,0.5}$ an irreducible real variety with $\dim X_{3,5,0.5} = 3$ and
$S_{3,5,0.5} = P_{3,5,0.5}$.
\end{thm}

\begin{thm}\label{thm:delta}
Let $b \ne \frac{3\,d^2 - a^2}{\sqrt{a^2 + 3d^2}}$ and $\dim X_{a,b,d} = 3$. Then the points in $P_{a,b,d}$ are not manifold points of $X_{a,b,d}$.
If $S_{a,b,d} = P_{a,b,d}$, it is $P_{a,b,d}$ the whole set of configuration space singularities of the Delta Manipulator. 
\end{thm}

\begin{remarks}\hfill
\begin{itemize}
\item[(i)]
It is conjectured, that the statement of theorem~\ref{thm:delta_jacobi_sv} is valid for all $a,b,d \in \R^+$, with $a > d$.

\item[(ii)]
We are at the moment not able to prove theorem~\ref{thm:delta} for the special case $b = \frac{3\,d^2 - a^2}{\sqrt{a^2 + 3d^2}}$, which
represents an exceptional posture for configurations $q_4$ and $q_3$ of table~\ref{table:cs_singularities}, but it is conjectured, that the theorem
is still valid in this cases.

\item[(iii)]
Theorem~\ref{thm:delta_jacobi_sv} can be easily proven by calculating a groebner base of the ideal $S$ generated by the $s_i,c_j,l_k$ and all principal minors of $F$. Since $S$ is a symmetric ideal with regard to the permutation $x_i/y_i/z_i/ca_i/sa_i \to x_{\pi(i)}/y_{\pi(i)}/z_{\pi(i)}/ca_{\pi(i)}/sa_{\pi(i)}$, with $\pi = (123)$, we can apply the algorithm from~\cite{steidel:symmetry}, though other modular methods (i.e. \texttt{modStd} of 'Singular') should work almost as well. We were able to calculate a groebner base for $S$ with a core-i5 laptop in less than 5 minutes. In this way we can proceed for other parameter of choice or with more work one should be able to derive a genericity statement for Theorem~\ref{thm:delta_jacobi_sv}.

\item[(iv)] 
It is not enough for the proof of theorem~\ref{thm:delta} to show, that $\algi(X_{a,b,d})$ is generated by the polynomial entries of $F$, since in real algebraic geometry singularities might still be manifold points (compare section~\ref{sec:prelim}). But it follows from theorem~\ref{thm:delta}, that all the points in $P_{a,b,d}$ are singularities of $X_{a,b,d}$ as real algebraic set.
\end{itemize}
\end{remarks}

\begin{table}[ht]
	\renewcommand{\arraystretch}{1.5}
	\caption{CS-Singularities of the Delta-Manipulator with $q=\sqrt{a^2 + 3\,d^2}$}\label{table:cs_singularities}
	\centering
		\begin{tabular}{@{}ccccc} \toprule
		Variable & $q_1$ & $q_2$ & $q_3$ & $q_4$ \\  \midrule
$x_{1}$ 		 & 	 $-  \frac{db}{\sqrt{a^2 + 3\,d^2}}$	 		& 	$\frac{db}{\sqrt{a^2 + 3\,d^2}}$ 			& 	$\frac{db}{\sqrt{a^2 + 3\,d^2}}$ 		 								&	$-  \frac{db}{\sqrt{a^2 + 3\,d^2}}$ \\
$y_{1}$ 		 & 	 $- \frac{\sqrt{3}\,db}{\sqrt{a^2 + 3\,d^2}} $ 		& 	$ \frac{\sqrt{3}\,db}{\sqrt{a^2 + 3\,d^2}}$		& 	$ \frac{\sqrt{3}\,db}{\sqrt{a^2 + 3\,d^2}}$ 		 							&	$- \frac{\sqrt{3}\,db}{\sqrt{a^2 + 3\,d^2}} $ \\
$z_{1}$ 		 & 	 $  \frac{\sqrt{a^2 - d^2} \, b}{\sqrt{a^2 + 3\,d^2}}$ 	& 	$\frac{\sqrt{a^2 - d^2} \, b}{\sqrt{a^2 + 3\,d^2}}$	& 	$\frac{\sqrt{a^2 - d^2} \, b}{\sqrt{a^2 + 3\,d^2}}$ 		 						&	$  \frac{\sqrt{a^2 - d^2} \, b}{\sqrt{a^2 + 3\,d^2}}$ \\
$x_{2}$ 		 & 	 $-  \frac{db}{\sqrt{a^2 + 3\,d^2}}$ 			& 	$\frac{db}{\sqrt{a^2 + 3\,d^2}}$			& 	$\frac{db}{\sqrt{a^2 + 3\,d^2}}$ 		 								&	$-  \frac{db}{\sqrt{a^2 + 3\,d^2}}$ \\
$y_{2}$ 		 & 	 $ \frac{\sqrt{3}\,db}{\sqrt{a^2 + 3\,d^2}}$ 		& 	$ -\frac{\sqrt{3}\,db}{\sqrt{a^2 + 3\,d^2}}$		& 	$ -\frac{\sqrt{3}\,db}{\sqrt{a^2 + 3\,d^2}}$ 		 							&	$ \frac{\sqrt{3}\,db}{\sqrt{a^2 + 3\,d^2}}$ \\
$z_{2}$ 		 & 	 $  \frac{\sqrt{a^2 - d^2} \, b}{\sqrt{a^2 + 3\,d^2}}$ 	& 	$  \frac{\sqrt{a^2 - d^2} \, b}{\sqrt{a^2 + 3\,d^2}}$	& 	$  \frac{\sqrt{a^2 - d^2} \, b}{\sqrt{a^2 + 3\,d^2}}$ 		 						&	$  \frac{\sqrt{a^2 - d^2} \, b}{\sqrt{a^2 + 3\,d^2}}$ \\
$x_{3}$ 		 & 	 $ \frac{2db}{\sqrt{a^2 + 3d^2}}$	 		& 	$ -\frac{2db}{\sqrt{a^2 + 3d^2}}$	 		& 	$\frac{2bd (bq - 2a^2 + b^2 + 3d^2)}{2b(a^2 - 3d^2) - q(a^2 + b^2)}$ 						&	$-\frac{2bd (bq + 2a^2 - b^2 - 3d^2)}{2b(a^2 - 3d^2) + q(a^2 + b^2)}$ \\
$y_{3}$ 		 & 	 $\scriptstyle{0}$ 			 		& 	$\scriptstyle{0}$ 					& 	$\scriptstyle{0}$ 		 										&	$\scriptstyle{0}$ \\
$z_{3}$ 		 & 	 $  \frac{\sqrt{a^2 - d^2} \, b}{\sqrt{a^2 + 3\,d^2}}$ 	& 	$  \frac{\sqrt{a^2 - d^2} \, b}{\sqrt{a^2 + 3\,d^2}}$	& 	$-\frac{b\sqrt{a^2 - d^2}(2bq - a^2 - b^2 + 6d^2)}{-2b(a^2 - 3d^2) + q(a^2 + b^2)}$ 				&	$\frac{b\sqrt{a^2 - d^2}(2bq + a^2 + b^2 - 6d^2)}{2b(a^2 - 3d^2) + q(a^2 + b^2)}$ \\
$ca_{1}$ 		 & 	 $\scriptstyle{- d}$ 					& 	$\scriptstyle{- d}$ 		 			& 	$\scriptstyle{-d}$ 		 										&	$\scriptstyle{-d}$ \\
$sa_{1}$ 		 & 	 $\scriptstyle{\sqrt{a^2 - d^2}}$			& 	$\scriptstyle{-\sqrt{a^2 - d^2}}$			& 	$\scriptstyle{-\sqrt{a^2 - d^2}}$	 									&	$\scriptstyle{\sqrt{a^2 - d^2}}$ \\
$ca_{2}$ 		 & 	 $\scriptstyle{- d}$ 					& 	$\scriptstyle{- d}$ 					& 	$\scriptstyle{-d}$		 										&	$\scriptstyle{-d}$ \\
$sa_{2}$ 		 & 	 $\scriptstyle{\sqrt{a^2 - d^2}}$			& 	$\scriptstyle{-\sqrt{a^2 - d^2}}$			& 	$\scriptstyle{-\sqrt{a^2 - d^2}}$	 									&	$\scriptstyle{\sqrt{a^2 - d^2}}$ \\
$ca_{3}$ 		 & 	 $\scriptstyle{- d}$ 					& 	$\scriptstyle{- d}$ 					& 	$\frac{6bd(a^2 - d^2)(b - q)}{- 2bq(a^2 - 3d^2) + q^2(a^2 + b^2) } \scriptstyle{- d}$	 			&	$\frac{6bd(a^2 - d^2)(b + q)}{2bq(a^2 - 3d^2) + q^2 (a^2 + b^2) } \scriptstyle{- d}$ \\
$sa_{3}$ 		 & 	 $\scriptstyle{\sqrt{a^2 - d^2}}$			& 	$\scriptstyle{-\sqrt{a^2 - d^2}}$			& 	$\frac{6bd^2 \sqrt{a^2-d^2} (q + 2b) }{-2bq(a^2 - 3d^2) + q^2(a^2 + b^2)} \scriptstyle{-\sqrt{a^2-d^2}}$ 	&	$\frac{6bd^2 \sqrt{a^2-d^2} (q - 2b)}{2bq(a^2-3d^2) + q^2(a^2 + b^2)} \scriptstyle{+\sqrt{a^2-d^2}}$ \\ \bottomrule
	\end{tabular}

\end{table}

\subsection{The proof of Theorem~\ref{thm:delta_jacobi}}\label{sec:proof1}
We have the following faithful representation $\Psi \colon D_3 = \langle r, s \rangle \hookrightarrow \GL(15,\R)$:
\[
r \mapsto \begin{pmatrix}
\zm & E & \zm & & & \\
\zm & \zm & E & & \smash{\raisebox{-1ex}{\huge$\mathbf{0}$}} & \\ 
E & \zm & \zm & & &\\
& & & \zm & e & \zm\\
& \smash{\raisebox{-1ex}{\huge$\mathbf{0}$}} & & \zm & \zm & e \\
& & & e & \zm & \zm
\end{pmatrix}, \quad
s \mapsto
\begin{pmatrix}
S & \zm & \zm & & & \\
\zm & S & \zm & & \smash{\raisebox{-1ex}{\huge$\mathbf{0}$}} & \\ 
\zm & \zm & S & & &\\
& & & s & \zm & \zm\\
& \smash{\raisebox{-1ex}{\huge$\mathbf{0}$}} & & \zm & s & \zm\\
& & & \zm & \zm & s
\end{pmatrix},
\]
where
\[
s \coloneqq \begin{pmatrix} 1 & 0\\ 0 & -1 \end{pmatrix}, \qquad
S \coloneqq \begin{pmatrix} 1 & 0 & 0\\ 0 & 1 & 0 \\ 0 & 0 & -1 \end{pmatrix},
\] 
and $E,e$ are the identity matrices in $\R^{3\times3}$ and $\R^{2\times2}$ respectively. Now let $\Phi$ denote the induced action on $R = \R[x_i,y_i,z_i,ca_j,sa_j]$, i.e. $\Phi(d)(f) = f(\Psi(d)\,\boldsymbol{x})$. We will show, that $\Phi(d)(I) = I$, for $d \in D_3$ and $I$ the ideal generated by the polynomials $s_i,c_j,l_k$,  $i,j=1,2,3$ and $k = 1, \ldots,6$. Then $D_3$ acts on $X=\algv(I)$ via $\Psi$.

With the permutation $\pi = (123) \in S_3$, we clearly have
\[
\Phi(r)(s_i) = s_{\pi(i)}, \ \Phi(r)(c_i) = c_{\pi(i)}, \ \Phi(s)(s_i) = s_i, \ \Phi(s)(c_i) = c_i.
\]
Now we consider the action on $R^3$ componentwise and use $A^2 = A^{-1}$. It is then $\Phi(r)v_i = v_{\pi(i)}$, hence
\begin{gather*}
\Phi(r) (v_1 - A \, v_2)  = v_2 - A \, v_3 = A^{-1} (A v_2 - A^{-1} v_3) = -A^{-1} (v_1 - A v_2 - (v_1 - A^{-1} v_3)), \\
\Phi(r) (v_1 - A^{-1} \, v_3) = v_2 - A^{-1} \, v_1 = - A^{-1} (v_1 - A v_2),\\ 
\Phi(s) (l_i) = l_i, \ \text{for $i=1,2,4,5$},  \\
\Phi(s) (l_3) = -l_3, \quad \Phi(s)(l_6) = -l_6.
\end{gather*}
So it follows $\Phi(d)(I) = I$ for all $d \in D_3$ and we have an action of $D_3$ on $X$.

Now let $J \leq \R[x]$ be the ideal of the principal minors of $DF$. We check, that $\Phi$ acts on $J$.
According to definition we have:
\begin{equation}\label{eq:minors1}
D (\Phi(d)F) = D \bigl(F(\Psi(d)\, \boldsymbol{x})\bigr) = DF(\Psi(d)\, \boldsymbol{x}) \cdot \Psi(d)
= \Phi(d) DF(\boldsymbol{x}) \cdot \Psi(d),
\end{equation}
where we also write $\Phi(d)$ for the induced action on $R^{12 \times 15}$ and $R^{12}$.

For a tuple $K = (i_1, \ldots i_{12}) \in \N^{12}$ with $1 \leq i_1 \leq i_2 \leq  \ldots i_{12} \leq 15$ and a Matrix $M \in M^{12 \times 15}$, 
we denote by $M_{(K)}$ the matrix comprised of the $K$ columns of $M$.
From \eqref{eq:minors1} we conclude $D(\Phi(d) F) \cdot \Psi(d)^{-1}_{(K)} = \Phi(d) DF_{(K)}$ and consequently 
because the action of $\Phi(d)$ on $R$ respects the ring structure, we have
\begin{equation}\label{eq:minors2}
\Phi(d) \, \det DF_{(K)} = \det \, (\Phi(d) DF_{(K)}) = \det \bigr( D(\Phi(d) F) \cdot \Psi(d)^{-1}_{(K)} \bigl)
\end{equation}
Since $\Phi$ acts linearly on the $\R$-vector space generated by the polynomials $s_i,c_j,l_k$, $i,j=1,2,3$ and
$k = 1, \ldots 6$. we have $\Phi(d)\, F = A_d \, \cdot F$, for $d \in D_3$ and a corresponding $A_d \in \R^{12\times 12}$, so it is
\begin{equation}\label{eq:minors3}
D (\Phi(d) F) = A_d \cdot DF.
\end{equation}
With \eqref{eq:minors2} and \eqref{eq:minors3} we get now
\[
\Phi(d) \, \det DF_{(K)} = \det A_d \cdot \det (DF \cdot \Psi(d)^{-1}_{(K)}). 
\]
Now we can use the Cauchy-Binet Formula and we have
\[
\Phi(d) \det DF_{(K)} =  \sum_{\substack{L = (j_1, \ldots, j_{12}) \in \N^{12}\\ 1 \leq j_1 < \ldots < j_{12} \leq 15}} \det DF_{(L)} \cdot \det \Psi(d)^{-1}_{(K)^{(L)}} \cdot \det A_d,
\]
where $\Psi(d)^{-1}_{(K)^{(L)}}$ means the matrix comprising the $L$ rows of $\Psi(d)^{-1}_{(K)}$. This shows $\Phi(d)(J) \subset J$ and consequently $\Psi(d) S_{a,b,d} = \Psi(d) \, \algv(I+J) \subset \algv(I+J) = S_{a,b,d}$.

We easily check, that $\Psi(D_3)$ acts freely on the orbits generated by the four points $q_1,q_2,q_3,q_4$ of Table~\ref{table:cs_singularities}. To complete the proof we only need to make sure, that the points are well defined and real for all $a,b,d$ with $a>d$ and fulfill all polynomials in $I+J$. We can check the second statement easily with a CAS (we used the \texttt{sympy} python library) and so it remains to show the first statement. For $q_1$ and $q_2$ this is clear. So we have to investigate the denominators in the coordinates of $q_3$ and $q_4$.

We will show, that $u := 2\,b\,(a^2 - 3\,d^2) + q\,(a^2 + b^2) \ne 0$ for all real $a,b,d>0$ with $a>d$ and the same statement for the other denominators will follow in the same way. If we consider $u$ as quadratic equation in $b$, the discriminant is 
\[
4\,(a^2 - 3\,d^2)^2 - 4\,q^2 a^2 = a^4 - 6\,d^2\,a^2 + 9d^4 - a^2\,(a^2 + 3\,d^2) = -9\,a^2\,d^2 + 9\,d\,^4 < 0.
\]
So there is no real zero for $u$.

\subsection{The proof of Theorem~\ref{thm:delta}}
For simplicity of notation, we consider the isomorphic system $\tilde{X}_{a,b,d}$ given by the polynomials  $s_i$,$c_j$, $\tilde{l}_{k}$, $i,j=1,2,3$ and $k=1,\ldots, 6$, where
\begin{align*}
\begin{pmatrix}
\tilde{l}_1\\\tilde{l}_2\\\tilde{l}_3
\end{pmatrix}
& = \begin{pmatrix}
d + ca_1 + x_1\\
y_1\\
sa_1 + z_1
\end{pmatrix}
- A 
\begin{pmatrix}
d + ca_2\\
0\\
sa_2
\end{pmatrix}
- 
\begin{pmatrix}
x_2\\y_2\\z_2
\end{pmatrix}, \\
\begin{pmatrix}
\tilde{l}_4\\\tilde{l}_5\\\tilde{l}_6
\end{pmatrix}
& = \begin{pmatrix}
d + ca_1 + x_1\\
y_1\\
sa_1 + z_1
\end{pmatrix}
- A^{-1}
\begin{pmatrix}
d + ca_3\\
0\\
sa_3
\end{pmatrix}
- 
\begin{pmatrix}
x_3\\y_3\\z_3
\end{pmatrix}.
\end{align*}
In addition we define:
\begin{equation*}
m_1(\psi) \coloneqq A
\begin{pmatrix}
d + a \cdot \cos(\psi) \\
0 \\
a \cdot \sin(\psi) 
\end{pmatrix},
\quad
m_2 (\psi) \coloneqq A^{-1} \, 
\begin{pmatrix}
d + a \cdot \cos(\psi) \\
0 \\
a \cdot \sin(\psi) 
\end{pmatrix},
\quad
m_3 (\psi) \coloneqq \, 
\begin{pmatrix}
d + a \cdot \cos(\psi) \\
0 \\
a \cdot \sin(\psi) 
\end{pmatrix}.
\end{equation*}
and $m(\psi) \coloneqq \bigl(a \cos(\psi), a \sin(\psi) \bigr)$. So we have the following simple characterization for points in $\tilde{X}_{a,b,d}$. 

\begin{lemma}\label{lemma:x_characterization}
Let $p = (x,y,z) \in \R^3$ and $\psi_i \in \R$, for $i=1,2,3$, then it is
\begin{equation}\label{eq:char_conf}
\begin{aligned}
 Q(p,\psi_1,\psi_2,\psi_3) & \coloneqq \bigl(p - m_1(\psi_1), p - m_2(\psi_2), p - m_3(\psi_3), m(\psi_1), m(\psi_2), m(\psi_3) \bigr) \in \tilde{X}_{a,b,d} \\
& \phantom{:::=} \text{iff $\abs{p - m_i(\psi_i)}^2 = b^2$, for $i=1,2,3$}. 
\end{aligned}
\end{equation}
\end{lemma}

Statement~\eqref{eq:char_conf} can be checked easily by verifying that $Q(p,\psi_1,\psi_2,\psi_3)$ fulfills all polynomials $c_i,s_j,\tilde{l}_k$, iff 
$\abs{p - m_i(\psi_i)}^2 = b^2$, for $i=1,2,3$. Furthermore it is clear, that every point
$q \in \tilde{X}_{a,b,d}$ can uniquely be represented in the form $q = Q(p,\psi_1,\psi_2,\psi_3)$, for a $p \in \R^3$ and $\psi_i \in \R$, $i=1,2,3$. We'll write $p(q)$ and $\psi_i(q)$ to reference those coordinates.

Finally, by abuse of notation we will use $q_1, \ldots, q_4$ to identify the points in the transformed configuration space $\tilde{X}_{a,b,d}$ corresponding to $q_i$ from table~\ref{table:cs_singularities}. We only need to show the theorem for $q_1, \ldots, q_4$ and will do this only for $q_1$ and $q_4$, since the assertion for $q_2$ and $q_3$ will follow analogously as in the case of $q_1$ and $q_4$ respectively. 

The idea of the proof is simple. We find four analytical paths $\gamma$ in $\tilde{X}$ with $\gamma(t_0) = q_i$ such that
the tangent vectors $\gamma'(t_0)$ will span a 4-dimensional subspace of $\R^{15}$, Since $\dim \tilde{X}_{a,b,d} = 3$ according to assumption, this is a contradiction and $\tilde{X}_{a,b,d}$ can't be 
a manifold locally at $q_i$. The constructed paths will offer some insights in the kinematic properties of the Delta Platform in the singular configurations.

\subsubsection{Singularities $q_4$ (and $q_3$)} 
We fix $\ph_i = \psi_i(q_4)$ for $i=1,2,3$ and $p_0 \coloneqq p(q_4)$. Then it is $\ph \coloneqq \ph_1 = \ph_2$ and $m_1(\ph_1) = m_2(\ph_2)$, $m_3(\ph_3)$ and $p_0$ all lie in the $xz$-plane. We will also define $S_i(t)$ to be the sphere with radius $b$ and center $m_i(t)$, so that
\[
p_0 \in S_1(\ph_1) \cap S_2(\ph_2) \cap S_3(\ph_3)
\]

\paragraph{First Path $\gamma_1$:} 
It is $S_1(\ph) = S_2(\ph)$ and for $t \ne \ph$ close to $\ph$ it is $K_1(t) \coloneqq S_1(t) \cap S_2(t)$ a circle in the $xz$-plane. We will denote the center of $K_1(t)$ with $M_1(t)$ and its radius with $r_1(t)$. Both $M_1(t)$ and $r_1(t)$ clearly admit analytic continuations for $t=\ph$ (You can find expressions for most terms in table~\ref{table:formula}).
\begin{table}[ht]
	\renewcommand{\arraystretch}{1.5}
	\centering
	\caption{geometric constraints}\label{table:formula}
		\begin{tabular}{@{}ccccc} \toprule
Variable & Term \\  \midrule
$p(q_4)$ & $\left(\frac{2bd}{\sqrt{a^2 + 3d^2}}, 0, \sqrt{a^2 - d^2} + \frac{b\sqrt{a^2 - d^2}}{a^2 + 3d^2} \right)^T$\\
$S_1(t)$ & $\left(-\frac{d}{2} - \frac{a \cos(t)}{2}, \sqrt{3}\cdot \frac{d + a \cos(t)}{2}, a \sin(t)\right)^T$\\
$S_2(t)$ & $\left(-\frac{d}{2} - \frac{a \cos(t)}{2}, -\sqrt{3}\cdot \frac{d + a \cos(t)}{2}, a \sin(t)\right)^T$\\
$S_3(t)$ & $\left(d + a \cos(t),0,a \sin(t) \right)^T$\\
$M_1(t)$ & $\left( -\frac{d}{2} - \frac{a \cos(t)}{2}, 0, a \sin(t) \right)^T$ \\
$r_1(t)$ & $\sqrt{-\frac{3 \, a^2 \cos^2(t)}{4} - \frac{3\, ad\cos(t)}{2} - \frac{3d^2}{4} + b^2}$ \\ \bottomrule
	\end{tabular}

\end{table}
As one can check quickly it is $m_3(\ph_3)$ the reflection of $m_1(\ph) = m_2(\ph)$ across the axis through the origin and $p$, as in figure~\ref{fig:reflection}.
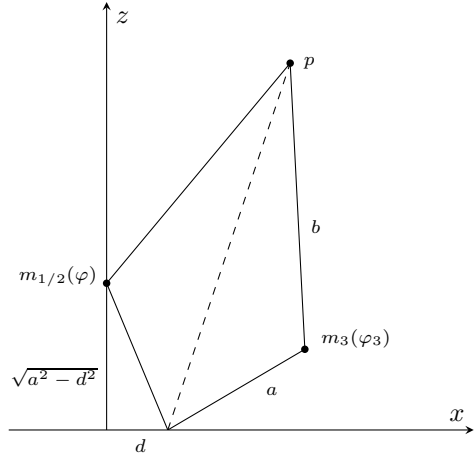
\begin{figure}[h]\label{fig:reflection}
	\centering
	\begin{tikzpicture}
\begin{axis}[xmin=-0.8,xmax=3, ymin=0, ymax=3.5, unit vector ratio*=1 1 1, axis lines=center, axis on top, ytick=\empty, xtick=\empty, xlabel={$x$}, ylabel={$z$}]

\addplot [-] coordinates{(0.5,0)(0,1.2)(1.5,3)};
\addplot [-] coordinates{(0.5,0)(1.62,0.66)(1.5,3)};
\addplot [-,dashed] coordinates{(0.5,0)(1.5,3)};


\node[label={0:{\scriptsize $p$}}, circle,fill,inner sep=1pt] at (axis cs:1.5,3) {};
\node[label={0:{\scriptsize $m_3(\ph_3)$}}] at (axis cs:1.6,0.75) {};
\node[circle,fill,inner sep=1pt] at (axis cs:0,1.2) {};
\node[circle,fill,inner sep=1pt] at (axis cs:1.62,0.66) {};
\end{axis}		
\node[anchor=north] at (1.75,0) {{\scriptsize $d$}};
\node[anchor=east] at (1.3,0.7) {{\scriptsize $\sqrt{a^2 - d^2}$}};
\node[anchor=east] at (1.3,2.05) {{\scriptsize $m_{1/2}(\ph)$}};
\node[anchor=north west] at (3.3,0.7) {{\scriptsize $a$}};
\node[anchor=south west] at (3.9,2.5) {{\scriptsize $b$}};
	\end{tikzpicture}\caption{Reflection of $m_3(\ph_3)$}
\end{figure}
This means, that the intersection of $K_1(\ph)$ with $S_3(\ph_3)$ is nonsingular as long as
\[
p - m_{1/2}(\ph) = \begin{pmatrix}
\frac{2bd}{\sqrt{a^2} + 3d^2} & 0 & \frac{b \sqrt{a^2 - d^2}}{\sqrt{a^2 + 3d^2}}
\end{pmatrix}^T
\]
is not perpendicular to the position vector of
\[
p - \begin{pmatrix}
d & 0 & 0
\end{pmatrix}^T = 
\begin{pmatrix}
\frac{2bd}{\sqrt{a^2} + 3d^2} - d & 0 & \sqrt{a^2 - d^2} + \frac{b \sqrt{a^2 - d^2}}{\sqrt{a^2 + 3d^2}}
\end{pmatrix}^T,
\] 
which is equivalent to
\[
b\,(a^2 + 3d^2) + \sqrt{a^2 + 3d^2}\,(a^2 - 3d^2) = 0,
\]
or
\begin{equation}\label{eq:perpendicular_intersection}
b = \frac{3d^2 - a^2}{\sqrt{a^2 + 3d^2}}.
\end{equation}
which we excluded in our assumption. 

Accordingly the intersection $K_1(\ph) \cap S_3(\ph_3)$ is nonsingular and due to the analytic implicit function theorem, we find an analytic
path $\delta_1 \colon (\ph - \eps, \ph + \eps) \to \R^3$, with 
\begin{equation}\label{eq:delta}
\delta_1(t) \in K(t) \cap S_3(\ph_3) \subset S_1(t) \cap S_2(t) \cap S_3(\ph_3).
\end{equation}
and $\delta_1(\ph) = p$. We now set $\gamma_1(t) \coloneqq Q(\delta_1(t),t,t,\ph^*)$ 
and according to \eqref{eq:delta} and \eqref{eq:char_conf} it is $\gamma_1(t) \in \tilde{C}_D$, for all $t$.
We immediately check $\gamma_1(\ph) = q_4$ and
\[
\gamma_1'(\ph) = \begin{pmatrix}
*\\
\vdots \\
* \\
-\sqrt{(a^2 - d^2)} \\
-d\\
-\sqrt{(a^2 - d^2)} \\
-d \\
0\\
0
\end{pmatrix}.
\]

\begin{remark}
It can be verified, that $\delta'(\ph) = 0$, so we are able to state $\gamma_1'(\ph)$ explicitly. This is not needed for the proof however.
\end{remark}	

\paragraph{Second Path $\gamma_2$:} Analogously to the first path, we find an analytic path 
$\delta_2 \colon (\ph_3-\eps, \ph_3 + \eps) \to \R^3$, with
\begin{equation}\label{eq:delta2}
\delta_2(t) \subset S_1(\ph) \cap S_2(\ph) \cap S_3(t),
\end{equation}
and $\delta_2(\ph_3) = p$. Hence $\gamma_2(t) \coloneqq Q(\delta_2(t),\ph,\ph,t) \in \tilde{C}_D$, $\gamma_2(\ph_3) = q_4$ and
we calculate: 
\[
\gamma_2'(\ph) = \begin{pmatrix}
*\\
\vdots \\
* \\
0 \\
0 \\
0 \\
0 \\
-\frac{6bd^2 \sqrt{a^2-d^2} (q - 2b)}{2bq(a^2-3d^2) + q^2(a^2 + b^2)} \scriptstyle{+\sqrt{a^2-d^2}} \\
\frac{6bd(a^2 - d^2)(b + q)}{2bq(a^2 - 3d^2) + q^2 (a^2 + b^2) } \scriptstyle{- d}
\end{pmatrix}.
\]
Note that the last two entries of $\gamma_2'(\ph_3)$ are well defined and not both zero since the sum of their squares is $1$.

\paragraph{Third Path $\gamma_3$:} 
Since we excluded the case \eqref{eq:perpendicular_intersection}, the intersection of $S_1(\ph) \cap S_3(\ph_3) = S_2(\ph) \cap S_3(\ph_3)$ must be a circle
which can be parameterized around $p$, i.e. it exists an analytical path $\delta_3 \colon (-\eps,\eps) \to \R^3$, with 
\[
\delta_3(t) \subset S_1(\ph) \cap S_2(\ph) \cap S_3(\ph_3),
\]
and $\delta_3(0) = p$. Clearly for $\gamma_3(t) \coloneqq Q(\delta_3(t),\ph,\ph,\ph_3) \in \tilde{C}_D$, it is $\gamma_3(0) = q_4$ and 
\[
\gamma_3'(0) = \begin{pmatrix}
*\\
\vdots \\
* \\
0 \\
0 \\
0 \\
0 \\
0 \\
0
\end{pmatrix}.
\]

\paragraph{Fourth Path $\gamma_4$:} We will now attempt to find a path in $S_1(t) \cap S_2(\ph) \cap S_3(\ph_3)$, for $t$ close to $\ph$.
Clearly it is $K_2(t) = S_1(t) \cap S_2(\ph)$ a circle for $t \ne \ph$ close to $\ph$. We denote the center of $K_2(t)$ as $M_2(t)$, the radius
of $K_2(t)$ as $r_2(t)$ and the normal 
\[
\frac{m_1(t) - m_2(\ph)}{\abs{m_1(t) - m_2(\ph)}}
\]
of the circle plane as $n(t)$. Some elementary considerations (compare Lemma~\ref{lemma:circle_intersection}) show, that $M_2(t)$, $r_2(t)$ and $n(t)$ admit an 
analytic continuation at $t = \ph$, with
\[
n(\ph) = M'(\ph) = m_1'(\ph), \quad r_2'(t) = 0.
\]
Now we show, that $p \in K_2(\ph) \cap S_3(\ph_3)$, where $K_2(\ph)$ means the circle associated to the analytic continuations of $M_2$,$r_2$ and $n$.
Since $p \in S_1(\ph) \cap S_3(\ph)$ and $K_2(\ph) \subset S_1(\ph)$ it suffices to show, that $p$ and $K_2(\ph)$ lie in the same plane. i.e.
$p - m_1(\ph) \perp n(\ph) = m_1'(\ph)$, but with $\ph = \pi - \arctan(\frac{\sqrt{a^2 - d^2}}{d})$ we see
\[
m_1'(\ph) = \begin{pmatrix}
\frac{\sqrt{a^2 - d^2}}{2} \\ - \frac{\sqrt{3}}{2} \, \sqrt{a^2 - d^2} \\ d
\end{pmatrix}^T
\]
and we easily check $(p - m_1(\ph)) \cdot m_1'(\ph) = 0$. 
Since we excluded the case \eqref{eq:perpendicular_intersection} it must be $K_2(t) \cap S_3(\ph_3)$ a nonsingular intersection for $t = \ph$. Because
$p \in K_2(\ph) \cap S_3(\ph_3)$ there must be an analytic path $\delta_4 \colon (\ph-\eps,\ph + \eps) \to \R^3$, with 
$\delta_4(\ph) = p$. Hence it is $\gamma_4(t) \coloneqq Q(\delta_4(t),t,\ph,\ph_3) \in \tilde{X}$ with $\gamma_4(\ph) = q_4$ and we check:
\[
\gamma_4'(\ph) = \begin{pmatrix}
*\\
\vdots \\
* \\
-\sqrt{(a^2 - d^2)} \\
-d \\
0 \\
0 \\
0 \\
0
\end{pmatrix}.
\]
Clearly we have $\dim \langle \gamma_1'(\ph), \gamma_2'(\ph_3), \gamma_3'(0), \gamma_4'(\ph) \rangle = 4$, what we wanted to show. 

\subsubsection{Singularities $q_1$ (and $q_2)$}
We fix again $\ph_i = \psi_i(q_1)$ for $i=1,2,3$ and $p_0 \coloneqq p(q_1)$.
It is 
\[
\ph_1 = \ph_2 = \ph_3 = \arctan\left(-\frac{\sqrt{a^2 - d^2}}{d}\right) + \pi 
\]
and
\[
m_1(\ph_1) = m_2(\ph_2) = m_3(\ph_3)
= \begin{pmatrix}
0\\
0\\
\sqrt{a^2 - d^2}.
\end{pmatrix}. 
\]

\paragraph{First and second path $\gamma_1$, $\gamma_2$:}
For $p \in S_1(\ph_1) = S_2(\ph_2) = S_3(\ph_3)$ it is obviously $Q(p,\ph_1,\ph_2,\ph_3) \in \tilde{X}$, hence we find paths $\gamma_1$,$\gamma_2$ in $\tilde{X}$,
with $\gamma_1(0) = \gamma_2(0) = q_1$ and
\[
\gamma_1'(0) = \begin{pmatrix}
0\\
1 \\
0 \\
0 \\
1 \\
0 \\
0 \\
1 \\
0 \\
0 \\
\vdots \\
0
\end{pmatrix}, \qquad
\gamma_2'(0) = \begin{pmatrix}
-\sqrt{a^2 - d^2}\\
0 \\
2\,d \\
-\sqrt{a^2 - d^2}\\
0 \\
2\,d \\
-\sqrt{a^2 - d^2}\\
0 \\
2\,d \\
0 \\
\vdots \\
0
\end{pmatrix}, \qquad
\]

\paragraph{Third Path $\gamma_3$:} 
Like in the first path it is $K_1(t) \coloneqq S_1(t) \cap S_2(t)$ a circle in the $xz$-plane. Now we donate with $K_2$ the circle given by the intersection
of $S_3(\ph_3)$ with the $xz$-plane. we have 
\[
M_1'(\ph_1) \cdot (p - M_1(\ph))= 
\begin{pmatrix}
\frac{\sqrt{a^2 - d^2}}{2} \\
0 \\
-d
\end{pmatrix}
\cdot 
\begin{pmatrix}
\frac{2db}{\sqrt{a^2 + 3d^2}}\\
0 \\
\frac{\sqrt{a^2 - d^2} \, b}{\sqrt{a^2 + 3\,d^2}}
\end{pmatrix}
= 0
\]
Hence, according to Lemma~\ref{lemma:circle_intersection} we find an analytic path
$\delta \colon (-\eps,\eps) \to \R^3$ with $\delta(t) \in K_1(t) \cap K_2(t)$ and 
$\delta(0) = p$. This means that $\gamma_3(t) \coloneqq Q(\delta(t),t+\ph,t+\ph,0) \in \tilde{X}$, with
\[
\gamma_3'(0) = \begin{pmatrix}
*\\
\vdots \\
* \\
-\sqrt{(a^2 - d^2)} \\
-d\\
-\sqrt{(a^2 - d^2)} \\
-d \\
0\\
0
\end{pmatrix}.
\]

\paragraph{Fourth Path $\gamma_4$:}
According to Corollary~\ref{cor:sphere_intersection} we can find an analytic path $\delta \colon (\ph-\eps,\ph+\eps) \to \R^3$, with $\delta(t) \in S_1(t) \cap S_2(\ph) \cap S_3(\ph)$ and $\delta(\ph) = p$,
if $p - m_1(\ph) \perp m_1'(\ph)$, but
\[
m_1'(\ph) \cdot (p - m_1(\ph)) = 
\begin{pmatrix}
\frac{\sqrt{a^2 - d^2}}{2} \\
\frac{\sqrt{3}}{2} \sqrt{a^2 - d^2}  \\
-d
\end{pmatrix}
\cdot
\begin{pmatrix}
\frac{2db}{\sqrt{a^2 + 3d^2}}\\
0 \\
\frac{\sqrt{a^2 - d^2} \, b}{\sqrt{a^2 + 3\,d^2}}
\end{pmatrix}
= 0.
\]

So we set again $\gamma_4(t) \coloneqq Q(\delta(t),t,0,0) \in \tilde{X}$, with 
\[
\gamma_4'(\ph) = \begin{pmatrix}
*\\
\vdots \\
* \\
-\sqrt{(a^2 - d^2)} \\
-d \\
0 \\
0 \\
0 \\
0
\end{pmatrix}.
\]
We now have $\dim \langle \gamma_1'(0), \gamma_2'(0), \gamma_3'(0), \gamma_4'(\ph) \rangle = 4$ again.
Since all non-manifold points are singularities, the 24 considered points give the full set of configuration space singularities, if $P_{a,b,d} = S_{a,b,d}$. This completes the proof of theorem~\ref{thm:delta}.


\begin{lemma}\label{lemma:circle_intersection}
Let $r,p_x,p_y\colon \R \to \R$ be analytic on a neighborhood of the origin, with
\[
r(0) := r_0 >0, \ r'(0) = 0 \quad \text{and} \quad (p_x(0), p_y(0)) = (0,0),\ (p_x'(0), p_y'(0)) \ne (0,0).
\]
We consider the intersection 
\begin{equation}\label{eq:circle_intersection}
\begin{aligned}
x^2 + y^2 - r_0^2 & = 0,\\
(x - p_x(t))^2 + (y - p_y(t))^2 - r(t)^2 & = 0,
\end{aligned}
\end{equation}
There exists two analytic paths $\gamma_{1/2}(t) = (x(t),y(t))$, $|t| < \eps$, fulfilling \eqref{eq:circle_intersection}, with
\[
\gamma_{1/2}(0) = \pm \frac{r_0}{\sqrt{p_x'(0)^2 + p_y'(0)^2}} \begin{pmatrix} 
-p_y'(0)\\
p_x'(0)
\end{pmatrix} =: b_{\pm}.
\]
\end{lemma}

\begin{proof}
We choose a coordinate system in such a way, that we can assume $p_x'(0) \ne 0$, $p_y'(0) \ne 0$. We set $d(t) \coloneqq \sqrt{p_x(t)^2 + p_y(t)^2}$ and $l(t) \coloneqq \frac{{r_0}^2 + p_x(t)^2 + p_y(t)^2 - r(t)^2}{2\,d(t)}$. 
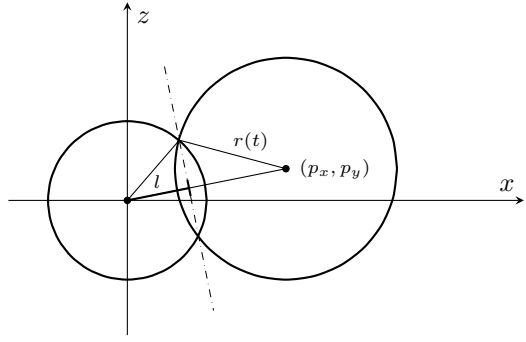
\begin{figure}[h]
	\centering
		\begin{tikzpicture}
		\begin{axis}[name=first, xmin=-1.5,xmax=5,ymin=-1.7, ymax=2.5, unit vector ratio*=1 1 1, axis lines=center, axis on top, ytick=\empty, xtick=\empty, xlabel={$x$}, ylabel={$z$}]
		
		
		\addplot [-] coordinates{(0,0) (2,0.4)};
		\addplot[smooth, thick, domain=0:360]%
		({cos(x)}, {sin(x)});
		\addplot[smooth, thick, domain=0:360]%
		({2+1.4*cos(x)}, {0.4 + 1.4*sin(x)});
		
		\addplot[] coordinates{(0,0) (0.655,0.76)};
		\addplot[] coordinates{(0.655,0.76)(2,0.4)};
		\addplot[very thin, dashdotted, shorten >= 7.6mm, shorten <= -10mm] coordinates{(0.655,0.76) (0.9,-0.46)};
		\addplot[very thin, dashdotted, shorten >= -10mm, shorten <= 7.6mm] coordinates{(0.655,0.76) (0.9,-0.46)};
		\addplot [-|, thick, shorten >= 13mm] coordinates{(0,0) (2,0.4)};
		\node[label={0:{\scriptsize $l$}}] at (axis cs:0.1,0.24) {};	
		
		\node[label={0:{\scriptsize $(p_x,p_y)$}}, circle,fill,inner sep=1pt] at (axis cs:2,0.4) {};
		\node[label={0:{\scriptsize $r(t)$}}] at (axis cs:1.1,0.75) {};
		\node[circle,fill,inner sep=1pt] at (axis cs:0,0) {};

		\end{axis}
		\end{tikzpicture}%
\end{figure}

\noindent Furthermore let (for $t \ne 0$)
\begin{align*}
q_{\pm}(t) \coloneqq & \frac{l(t)}{d(t)} \cdot \begin{pmatrix} p_x(t) \\ p_y(t) \end{pmatrix} \pm \frac{\sqrt{r_0^2 - l(t)^2}}{d(t)} \cdot 
\begin{pmatrix}
-p_y(t) \\ p_x(t)
\end{pmatrix} \\
 = & -\frac{1}{2} \left( \frac{r_0^2 - r(t)^2}{d(t)^2} + 1 \right) \begin{pmatrix} p_x(t) \\ p_y(t) \end{pmatrix} \pm  \frac{\sqrt{r_0^2 - l(t)^2}}{d(t)} \cdot 
\begin{pmatrix}
-p_y(t) \\ p_x(t)
\end{pmatrix}.
\end{align*}
One easily checks, that $q(t)_{\pm}$ fulfills the system \eqref{eq:circle_intersection}. 
We will need to show, that $q(t)$ can be continued analytically around $0$, that $q(t) \in \R$, for $t$ small enough, and that
either $q_+(t) \to b_+$ and $q_-(t) \to b_-$, for $t \to 0$. 

Since $p_x,p_y$ and $r$ are analytical around $0$ we can extend them to holomorphic 
functions on a small neighborhood of $0$ in $\C$. We set:
\begin{align*}
f_1(z) & \coloneqq 
\frac{r_0^2 - r(z)^2}{p_x(z)^2 + p_y(z)^2}.\\
f_2(z) & \coloneqq l^2(z).\\
f_3(z) & \coloneqq \frac{p_x^2(z)}{p_x^2(z) + p_y^2(z)}.\\
f_4(z) & \coloneqq \frac{p_y^2(z)}{p_x^2(z) + p_y^2(z)}.
\end{align*}
Since the origin can't be a limit point for the zeros of $p_x(z)^2 + p_y(z)^2$, we find an $\eps > 0$, such that $f_1, \ldots f_4$ are analytic on $B_{\eps} \backslash 0$.
Assume we have shown, that $f_1, \ldots, f_4$ admit analytic continuations on $B_{\eps}$, for which we write $f_1, \ldots f_4$ again. 
Assume also, that $f_1(0) =: b \in \R$, $f_2(0) = 0$ and $f_3(0), f_4(0) > 0$. With the main branch of logarithm we can define analytic functions: 
\begin{align*}
g_2(z) &\coloneqq \sqrt{r_0 - f_2(z)},\\
g_3(z) &\coloneqq \sqrt{f_3(z)},\\
g_4(z) & \coloneqq \sqrt{f_4(z)}.
\end{align*}
Now we can assume, that $g_3(t) = \frac{p_x(t)}{\sqrt{p_x(t) + p_y(t)^2}}$ and $g_4(t) = \frac{-p_y(t)}{\sqrt{p_x(t) + p_y(t)^2}}$, for $t \in \R$, otherwise
multiply with $-1$. Hence it is 
\[
q_{\pm}(t)
= -\frac{1}{2} (f_1(t) + 1) \begin{pmatrix}
p_x(t)\\
p_y(t)
\end{pmatrix}
\pm g_2(t)
\begin{pmatrix}
g_4(t)\\
g_3(t)
\end{pmatrix}.
\]
an analytic function for $t$ small enough. As $f_2(z) \to 0$ for $z \to 0$ and $r_0 > 0$, we have $g_2(t) \in \R$, for $t$ small enough and it follows
$q_{\pm}(t) \in \R$ for $t$ small enough. Moreover it is
\begin{align*}
\lim_{t \to 0} q_\pm (t) 
& = -\frac{1}{2} (b +1) \begin{pmatrix}
0\\
0
\end{pmatrix}
\pm g_2(0) \begin{pmatrix}
g_4(0)\\ 
g_3(0)
\end{pmatrix}
\end{align*}
We will see shortly, that $g_3(0) = \frac{p_x'(0)}{\sqrt{p_x'(0)^2 + p_y'(0)^2}}$
$g_4(0) = \frac{-p_y'(0)}{\sqrt{p_x'(0)^2 + p_y'(0)^2}}$, hence 
\[
\lim_{t \to 0} q_\pm (t) = \pm r_0 \cdot \frac{1}{\sqrt{p_x'(0)^2 + p_y'(0)^2}}
\begin{pmatrix}
-p_y'(0)\\
p_x'(0)
\end{pmatrix}.
\]
We still need to show, that $f_1, \ldots f_4$ admit analytic continuations on $B_{\eps}$ with values according to our earlier assumption. We will do
that only for $f_3$, since we can show the rest similar. Since
\begin{gather*}
(p_y^2 + p_x^2)''(0)  = 2 \, (p_x'(0)^2 + p_y'(0)^2)  + 2 p_x(0)\,p_x''(0)^2 + 2 p_y(0)\,p_y''(0)^2 = 2 \, (p_x'(0)^2 + p_y'(0)^2) > 0, \\
(p_x^2)''(0)  = 2 \, p_x'(0)^2 + 2 \, p_x(0)\,p_x''(0) = 2 \, p_x'(0)^2 > 0,
\end{gather*}
it is
\[
\lim_{z \to 0} \frac{p_x^2(z)}{p_x^2(z) + p_y^2(z)} = 
\frac{2\,p_x'(0)^2}{2\,(p_x'(0)^2 + p_y'(0)^2)} = \frac{p_x'(0)^2}{p_x'(0)^2 + p_y'(0)^2} > 0.
\]
This means that $f_3(z)$ is holomorphic on $B_{\eps} \backslash 0$ and admits a continuous continuation on $B_{\eps}$, but then it can be continued analytically on $B_\eps$.
\end{proof}

\begin{cor}\label{cor:sphere_intersection}
Consider the intersection of spheres
\begin{equation}\label{eq:sphere_intersection}
\begin{aligned}
x^2 + y^2 + z^2 - r_0 & = 0\\
(x - p_x(t))^2 + (y - p_y(t))^2 + (z - p_z(t))^2- r(t) & = 0,
\end{aligned}
\end{equation}
where $p_x,p_y,p_z,r_1$ analytic, $r'(0) = 0$, $r(0) = r_0$ and
\[
\begin{pmatrix}
p_x(0)\\
p_y(0)\\
p_z(0)
\end{pmatrix}
= \begin{pmatrix}
0\\0\\0
\end{pmatrix}, \quad
\begin{pmatrix}
p_x'(0)\\
p_y'(0)\\
p_z'(0)
\end{pmatrix}
\ne
\begin{pmatrix}
0\\0\\0
\end{pmatrix}.
\]
For every $p \in B_{r_0}(0)$ with $p \perp (p_x'(0), p_y'(0), p_z'(0))^T$ there exists an analytic path
$\delta \colon (-\eps,\eps) \to \R^3$ fulfilling \eqref{eq:sphere_intersection}, with $\delta(0) = p$.
\end{cor}

\begin{proof}
Let $p \in B_{r_0}$ with $p \perp (p_x'(0), p_y'(0), p_z'(0))^T$ and let $E$ be the plane through the origin spanned by the position vector of $p$ and $(p_x'(0), p_y'(0), p_z'(0))^T$. We choose $E$ as new 2-dimensional coordinate system and the statement follows with lemma~\ref{lemma:circle_intersection}
applied to circles given by the intersection of E with the spheres.
\end{proof}

\bibliography{literatur}{}
\bibliographystyle{plain}
\end{document}